\documentclass[11pt]{article}

\usepackage[preprint]{acl}

\usepackage{times}
\usepackage{latexsym}

\usepackage[T1]{fontenc}

\usepackage[utf8]{inputenc}

\usepackage{microtype}

\usepackage{inconsolata}

\usepackage{graphicx}
\usepackage[utf8]{inputenc} 
\usepackage[T1]{fontenc}    
\usepackage{hyperref}       
\usepackage{url}            
\usepackage{booktabs}       
\usepackage{amsfonts}       
\usepackage{nicefrac}       
\usepackage{microtype}      
\usepackage{xcolor}         

\usepackage{amsmath}
\usepackage{graphicx}
\usepackage{array}
\usepackage{amsthm}

\newtheorem{proposition}{Proposition}

\newtheorem{theorem}{Theorem}
\newtheorem{corollary}{Corollary}

\title{Know Thy Reasoner: Not All Language Models Explore Alike}

\author{Moulik Choraria \thanks{Correspondence: moulikc2@illinois.edu; Initial formulation was done during MC's internship at Capital One} \\
UIUC\\
\And
Argyrios Gerogiannis \\
UIUC \\
\And
Anirban Das \\
Capital One\\
\And
Supriyo Chakraborty \\
Capital One\\
\AND
Sourya Basu \\
Capital One \\
\And
Sambit Sahu\\
Capital One\\
\And
Lav R. Varshney\\
Stony Brook University\\
}

\begin{document}
\maketitle
\begin{abstract}
Compute scaling for LLM reasoning trades off exploring solution approaches
(\emph{breadth}) against refining promising ones (\emph{depth}), yet why a given
trade-off works, and why it often fails to transfer across models, remains
unclear. We argue that \textbf{the optimal strategy depends on the model's
\emph{diversity profile}, the spread of probability mass across solution
approaches, and that this must be characterized before any exploration strategy
is adopted.} We formalize this with a framework decomposing reasoning
uncertainty, deriving when depth-based refinement outperforms parallel sampling,
and validate it across three model families at both inference and training. Our
central finding is that the diversity regime dictates the strategy: low-diversity aligned
models benefit from depth-based refinement with lightweight intrinsic signals,
whereas high-diversity base models are often harmed by it, and instead need breadth or stronger signals to compensate.

\end{abstract}

\section{Introduction}

The landmark success of DeepSeek-R1~\citep{Guo_et_al_deepseekr1_2025} in exploiting exploration for open-ended math reasoning has spurred a wave of follow-ups proposing alternate strategies for training-time exploration beyond basic
parallel sampling~\citep{xuSFK2025, zhuangZGHLSZ2025, yaoHZDXJZ2025}. Many of these rely on \emph{intrinsic} signals computed from the model's own
generations, such as token entropy~\citep{houHLLTD2025}, self-certainty or output confidence~\citep{zhaoKFLS25, yoonYJEDLJY2025}, and
self-critique~\citep{madaanTGHG_et_al_2023}. Such intrinsic signals are appealing because they require no additional supervision or auxiliary model: external verifiers~\cite{brown_verifiers_2025} and process reward models~\cite{zhang-etal-2025-lessons, Awesome-Process-Reward-Models} can be expensive to train and increase latency, be difficult to specify outside narrow domains, and prone to reward hacking, whereas intrinsic signals are available from the policy at no extra cost. 

This appeal has produced a steady stream of proposed signals, often evaluated
via a common template: propose a signal with some intuitive motivation,
evaluate it on one or two chosen models, and compare against a parallel-sampling
baseline such as GRPO or DAPO~\citep{shaoWZXS_et_al2024, yuZZYZYD_et_al_2025dapo}.
Yet despite the empirical gains, evaluation remains narrow, and it is unclear
why a given strategy helps on one model yet fails on another. The same
exploration strategy that improves performance on one checkpoint can leave
another unchanged or even degrade it (see Table~\ref{tab:pass_at_n}). We argue
that this inconsistency is not noise but a symptom: the effectiveness of an
exploration strategy depends on the model's \emph{diversity profile}, the spread
of probability mass across distinct solution approaches, and this regime must be
characterized before any strategy is adopted. Concretely, we ask how far
intrinsic, training-free signals can take exploration, and show that the answer
depends on the diversity regime of the model (with respect to the data/task). This reframes choices that current work leaves implicit, such as starting from base versus aligned checkpoints, as consequences
of where a model sits in this regime rather than arbitrary design decisions. To this end, we make the following contributions:
\begin{enumerate}
  \item We develop a \textbf{theoretical framework} (\S\ref{sec:theory}) that
        decomposes reasoning uncertainty and derives conditions, as a function
        of a model's diversity profile, under which depth-based refinement
        outperforms parallel sampling.
  \item We \textbf{validate these conditions at inference}
        (\S\ref{sec:experiments}) across three model families, showing that the
        predicted regime boundaries hold and that existing approaches lose
        effectiveness outside their regime. We further \textbf{measure the
        diversity profile directly} via lexical and semantic metrics, confirming
        the regimes are real rather than only inferred from performance.
  \item We show this behavior \textbf{persists during training}
        (\S\ref{sec:training}), where diversity again governs which strategies
        are effective.
\end{enumerate}
While diversity profile may not be the sole factor, since dataset choice, prompt design, model scale and something as trivial as sampling parameters can all confound the picture, as we discuss in \S\ref{sec:discussion}, these results offer strong support for our argument that the target model's diversity regime must be characterized before any exploration strategy is adopted for a given dataset.


\section{Related Work}
\label{sec:related}

\paragraph{Exploration for reasoning.}
Scaling inference-time compute through sampling is a well-established route to
stronger reasoning~\citep{weiWSBIXCLZ2023, shaoWZXS_et_al2024, snellLXK2024,
welleckBFSXNKH2024}. The simplest approach samples $N$ trajectories in parallel
and selects among them~\citep{cobbeKBCJKPTHNHS2021, wangWSLCNCZ2023}, while
tree-search methods expand and refine trajectories through depth-based
exploration~\citep{ZhaoLH2023, zhangZHYDDT2024}. These define a
breadth-versus-depth axis~\citep{mohrNR2025} that we formalize in \S\ref{sec:theory}. A separate
line of work studies exploration during training~\citep{Guo_et_al_deepseekr1_2025, xuSFK2025, zhuangZGHLSZ2025,
yaoHZDXJZ2025}, typically against parallel-sampling baselines such as GRPO and
DAPO~\citep{shaoWZXS_et_al2024, yuZZYZYD_et_al_2025dapo}.

\paragraph{Intrinsic and external signals.}
Exploration strategies differ in the feedback they use. External signals come
from trained verifiers or process reward models that provide dense, per-step
supervision~\citep{ghimire2026prism,zhangZHYDDT2024, uesatoKKSSW_et_al2022}. Intrinsic signals
are computed from the model's own outputs and require no additional training:
examples include output confidence or self-certainty~\citep{zhaoKFLS25, yoonYJEDLJY2025},
answer-probability rewards~\citep{yuJWYWCY_et_al2025}, consistency across
samples~\citep{zhangYLWLYST2025}, count-based novelty~\citep{zhangLZLQLSTQQ2025}, and self-refinement from
model-generated critiques~\citep{madaanTGHG_et_al_2023}. Several methods add
diversity-promoting intrinsic rewards to counteract the exploration loss induced
by reward optimization~\citep{huZLYHCQCCW2026, sunCWSWW2025}. Intrinsic signals are attractive because they require no auxiliary model and are available directly from the policy. We limit ourselves to the same regime, to isolate how far such signals can carry exploration.

\paragraph{Diversity and uncertainty in generation.}
A growing body of work shows that alignment reduces output diversity. RLHF
improves out-of-distribution generalization but measurably narrows the range of
generated outputs relative to supervised fine-tuning~\citep{kirkMNLHGR2024}, and similar diversity loss has been documented for conceptual~\citep{murthyUH2024} and preference-based~\citep{slocumSM2025} settings, often attributed to the KL regularizer in preference learning~\citep{slocumSM2025}. In the reasoning setting specifically, \citet{yueCLZWYSH2025} show that RLVR concentrates the output distribution toward high-reward paths, improving pass@$1$ while leaving base models with comparable
or higher pass@$k$, indicating that RL narrows rather than expands the reachable
set of approaches. \textbf{This base-versus-aligned contrast is the empirical basis for
the diversity regimes we study}. To characterize uncertainty in reasoning, we
follow the decomposition of \citet{bakmanKHYB_et_al_2025}, separating sources of
uncertainty rather than treating it as a scalar. We measure response-level
diversity directly using the Vendi score~\citep{friedman2023vendiscore} and
dense embeddings~\citep{chen2024bgem3}, reported in \S\ref{sec:experiments}.

\section{Theory}
\label{sec:theory}

\subsection{A Budget Allocation Perspective}

To understand the role of diversity, we first delve into what lies at the core of exploration: the trade-off between \emph{breadth} and
\emph{depth}. For a fixed model and compute/token budget, how should a
strategy divide effort between the two? At one extreme, standard parallel
sampling allocates all compute to breadth; at the other, depth-first refinement resembles tree search which requires a signal (for our purposes, intrinsic) to refine against. Model diversity is what ties these together: how a model responds to these signals is governed by its diversity profile, and, this profile shapes the optimal exploration strategy. We make this precise by first decomposing reasoning uncertainty~\citet{bakmanKHYB_et_al_2025}, and then formalizing the breadth-depth trade-off through an Monte Carlo Tree Search (MCTS)~\citep{wiechowskiGSM2022} abstraction.

\textbf{Setup}: To make the breadth-depth distinction precise we need a formal notion of a solution \emph{approach}, since breadth is exploration \emph{across} approaches and depth is refinement \emph{within} one. Let $q$ denote the problem of interest; in what follows we condition on $q$ throughout. A model parametrized
by $\phi$ generates reasoning trajectories $\tau \sim p_\phi(\cdot)$, and we
write $\mathcal{T}$ for the set of all such trajectories. We assume $\mathcal{T}$
admits a \emph{partition} into $m$ disjoint measurable sets $\mathcal{T} =
\biguplus_{i=1}^m \mathcal{A}_i$, where each $\mathcal{A}_i$ is a \emph{latent
approach}, a class of high-level solution strategies such as analytical or
induction-based reasoning. For instance, on a problem
asking for a closed form of a finite sum, one approach $\mathcal{A}_i$ might
collect all trajectories that proceed by induction on $n$, another all
trajectories that use a generating-function argument, and a third those that
telescope the sum directly; each $\mathcal{A}_i$ groups trajectories that share
a high-level strategy while differing in their detailed execution. To characterize the probability of finding a correct
trajectory, we define:
\begin{align}
\pi_i(\phi) &= \mathbb{P}(\tau\in\mathcal{A}_i|\tau\sim p_{\phi}(\cdot)) \nonumber\\ 
&\quad \text{\small(probability that sample belongs to an approach)} \nonumber\\
\theta_i(\phi, B) &= \Pr(\text{correct} \mid \tau \in \mathcal{A}_i;\, \text{budget }B) \nonumber\\
&\quad \text{\small(model success rate under an approach)} \nonumber\\
\bar{\theta}(\phi, B) &= \textstyle\sum_{i} \pi_i \, \theta_i \nonumber\\
&\quad \text{\small(marginal model success rate)} \nonumber
\end{align}

Next, we define the the set of \emph{viable} approaches $\mathcal{V} = \{i : \exists\ \tau \in \mathcal{A}_i, \tau \text{ is correct} \}$ i.e.\ approaches that contain at least one correct trajectory. We also define an ideal model $\phi^*$, with the intuition being that it allocates masses to different approaches optimally via $\pi(\phi^*)$. We can now characterize the difficulty of the problem $q$ by decomposing the reasoning uncertainty into three components.\\
(i) \textit{Aleatoric uncertainty:} Let $p_i = \Pr(i \in \mathcal{V})$ be the probability that approach $i$ is viable for $q$. We define
$U_{\text{alea}}(q) = \sum_{i=1}^m H\bigl(\mathrm{Bern}(p_i)\bigr),$
the entropy of identifying the viable approaches among all $m$. It is maximized when each approach is equally viable or not (maximal ambiguity about which strategies solve $q$) and vanishes when viability is certain. It depends only on $q$ and is irreducible.\\
(ii) \textit{Epistemic breadth:} $U_{\text{epi}}^{\text{breadth}}(\phi) = D_{\text{KL}}\!\left(\pi(\phi^*) \,\|\, \pi(\phi)\right)$, which measures the mismatch in approach selection compared to the ideal model--analogous to the excess rate incurred in mismatched rate-distortion when encoding under the wrong prior~\citep{lapidoth1997}. We hypothesize that reducing this requires two stages: first, \emph{diverse sampling} (breadth exploration) to increase coverage of approaches in support of $\pi(\phi^*)$; then, \emph{importance-weighted selection} among approaches to correct mismatch between $\pi(\phi)$ and $\pi(\phi^*)$. \\
(iii) \textit{Epistemic depth:}
\begin{multline*}
U_{\text{epi}}^{\text{depth}}(\phi, B) = \sum_{i \in \mathcal{V}} \pi_i(\phi^*)\, \times \\
    D_{\text{KL}}\!\bigl(\mathrm{Bern}(\theta_i(\phi^*, B)) \,\|\, \mathrm{Bern}(\theta_i(\phi, B))\bigr)
\end{multline*}
Within each viable approach $i \in \mathcal{V}$, the model succeeds with
probability $\theta_i(\phi, B)$, which we compare against the ideal model's
success rate $\theta_i(\phi^*, B)$ by treating each as a Bernoulli outcome. The
KL divergence between the two Bernoullis measures the per-approach execution
gap: how much worse the model is at \emph{carrying out} an approach it has
already chosen, as opposed to choosing it. Weighting each gap by
$\pi_i(\phi^*)$, the ideal model's mass on that approach, and summing over
viable approaches gives an aggregate measure of how much of the model's failure
is attributable to execution rather than approach selection. Our hypothesis is
that this term complements breadth exploration and can be reduced by
\emph{feedback and iterative refinement}. With this framework in place, we can
now focus on designing an optimal exploration strategy.

\textbf{An MCTS Allocation Strategy.}
Consider a budget of $N$ trajectories. We split the budget into
$N_e = \alpha N$ for \emph{exploration} (discovering distinct approaches) and
$N_f = (1-\alpha)N$ for \emph{refinement} (improving within the most promising
discovered approach $i^*$). This mirrors MCTS: $N_e$ expands the tree across
approaches while $N_f$ concentrates on high-value branches. The question is:
under what conditions does this outperform i.i.d.\ sampling?

\begin{proposition}\label{prop:mcts}
Let $c_i := -\log(1-\theta_i)$ measure the per-sample refinement value of
approach $i$, $c_{\mathrm{rand}} := -\log(1-\bar{\theta})$ that of random
sampling, and $R^{(t)} := c_{i^*}^{(t)} / c_{\mathrm{rand}}$ the quality
ratio of the best approach over the marginal. Then, assuming oracle identification the optimal approach $i^*$:

\noindent (a) $\mathrm{pass@}N_{\mathrm{mcts}} \ge
\mathrm{pass@}N_{\mathrm{rand}}$ iff
$\alpha \le 1 - 1/R^{(t)}$.

\noindent (b) Discovering $i^*$ with probability $\ge 1-\eta$ under i.i.d.\
exploration requires
$\alpha \ge \alpha_{\min}(\eta) := \frac{\log(1/\eta)}{N \log(1/(1-\pi_{i^*}))}
\approx \frac{\log(1/\eta)}{N\pi_{i^*}}$.

\noindent Both conditions are jointly satisfiable iff
$R^{(t)} \ge \frac{1}{1-\alpha_{\min}(\eta)}$.
\end{proposition}

The proposition reveals a tension between two competing constraints. Condition
(a) imposes an upper bound on the exploration fraction: allocating too large an
$\alpha$ diverts budget from refinement, where it would be more effective.
Condition (b) imposes a lower bound: discovering rare approaches requires
sufficient exploration. The two are jointly satisfiable only when the quality
ratio $R^{(t)}$ is large enough to absorb the discovery cost. We note that this
analysis assumes all refinement concentrates on a single best approach $i^*$,
which represents the most favorable setting for depth-based strategies. Since
MCTS cannot outperform random sampling under any weaker refinement allocation
if it fails to do so here, the conditions in Proposition~\ref{prop:mcts} are
necessary for MCTS to be viable, and remain informative under this
simplification. The proof is included in Appendix \ref{app:proof_mcts}.

\subsection{On Approach Switching: "Wait..."}
\label{sec:switching}
Note that our breadth/depth framework treats an approach as committed at sampling time. In
practice, aligned reasoning models demonstrate the remarkable ability to (seemingly) sense their own uncertainty and revise mid-generation: they begin along one
approach, recognize that the derivation is unproductive, and backtrack to
another. This behavior is especially pronounced in RLVR-style models, whose
generations contain explicit self-correction (the characteristic ``wait''
before a revision). Modeling \emph{approach switching} therefore makes our
framework self-contained.

For simplicity, one could begin modelling approach switching as unconditional: at each step the model keeps its approach with probability $1-\rho$ and otherwise switches uniformly. This is
analytically clean (Appendix~\ref{app:proof_approach_switching}, Theorem~\ref{thm:unconditional_switching}) but
unrealistic, since it switches blindly rather than in response to evidence. Real
backtracking is \emph{failure-conditioned}: the model switches only after it gets an inkling that the
current attempt fails. We adopt this model. At each attempt the active approach
$i$ succeeds with probability $\theta_i$; on failure, the model keeps approach
$i$ with probability $1-\rho$ and otherwise switches uniformly to another
approach. The probability of solving within $H$ attempts (since we assume limited token budget) admits an exact form (Appendix~\ref{app:proof_approach_switching}, Theorem~\ref{thm:failure_conditioned_switching}),
but it is uninformative about \emph{when} switching helps. The first nontrivial
case is revealing. 

\begin{proposition}
\label{prop:two_step_switching_gain}
For $H=2$, the gain of failure-conditioned switching over staying with the same
approach is
\[
P^{\mathrm{bt}}_2(\rho)-P^{\mathrm{bt}}_2(0)
=
\rho
\sum_{i=1}^m
\pi_i(1-\theta_i)\left(\bar\theta_{-i}-\theta_i\right),
\]
where $\bar\theta_{-i} := \frac{1}{m-1}\sum_{j\neq i}\theta_j$ is the average
success probability of approaches other than $i$. Consequently, one-step
backtracking is beneficial iff $\sum_{i=1}^m \pi_i(1-\theta_i)(\bar\theta_{-i}-\theta_i) > 0$.
\end{proposition}

The weight $\pi_i(1-\theta_i)$ is proportional to the posterior probability that
a failed attempt came from approach $i$. Proposition~\ref{prop:two_step_switching_gain}
thus says backtracking helps when failure is evidence that the current approach
has below-average success relative to the alternatives. This captures the useful
form of RLVR-style self-correction: switching is valuable not because it blindly
increases diversity, but because failure is informative about which approach to
abandon. This is not merely a property of unprompted generation: test-time
scaling methods deliberately encourage such revision, suppressing the
end-of-thinking token and injecting a ``Wait'' token to push the model to
revisit and revise its reasoning~\citep{muennighoffYSLFH_et_al2025}, an intervention that
can also be learned rather than hand-specified~\citep{ringelTR2026}. Our
analysis predicts when such interventions help, giving a precise condition under
which forcing a revision is rational rather than wasteful. A further
specialization (Appendix~\ref{app:proof_approach_switching}, Corollary~\ref{cor:one_good_approach})
shows the threshold can exceed $1/m$, so backtracking can help even when the
model already selects the best approach more often than a uniform sampler, since
conditioning on failure shifts weight toward the weaker approaches.

\section{Experiments}
\label{sec:experiments}

We organize our empirical study in three parts. We first test the breadth/depth
prediction at \emph{inference}, measuring how base, instruct, and RLVR
checkpoints respond to depth-based refinement. We then characterize the
underlying \emph{diversity} profiles directly, confirming that the regimes our
theory relies on are measurable rather than only inferred from performance.
Finally, we ask whether the same pattern governs exploration during
\emph{training}. We do not empirically test the approach-switching model: while
elegant, identifying a genuine switch in practice is difficult, since a
``Wait'' token may signal a real change of approach or mere rumination, and
distinguishing these (along with other forms of backtracking) would require an
external LLM judge. We defer this to future work. Experiments were conducted on a combination of a single Nvidia-DGX-A100 node, as well as A100/H100/H200 nodes (2-4 GPUs) of a shared academic cluster. 

\subsection{Inference}

\begin{table*}[!t]
\centering
\small
\caption{Pass@$16$ for \textbf{Baseline}, \textbf{ENT} and \textbf{SR}. The Lift showcases the gain over the \textbf{Baseline}. 
}
\vspace{-4pt}
\begin{tabular}{@{}lccccc@{}}
\toprule
\textbf{Model} & \textbf{Baseline} & \textbf{ENT} & \textbf{SR} & \textbf{ENT Lift}  & \textbf{SR Lift}\\
 \midrule
 Qwen-3 4B base     & 0.9004 & 0.6377 & 0.6406 & -29.18\% & -28.85\% \\
 Qwen-3 4B instruct & 0.9121 & 0.9180 & 0.9170 & 0.65\% & 0.54\% \\
 Qwen-3 4B RLVR     & 0.6885 & 0.7471 & 0.7373 & 9.34\% & 7.90\% \\
 \midrule
 OLMo-3 7B base     & 0.8506 & 0.9072 & 0.8828 & 6.65\% & 3.79\% \\
 OLMo-3 7B instruct & 0.6797 & 0.8154 & 0.8086 & 19.96\% & 18.96\% \\
 OLMo-3 7B RLVR     & 0.7201 & 0.8300 & 0.7813 & 15.26\% & 8.50\% \\
 \midrule
 Nemotron-3 30B base & 0.9277 & 0.7734 & 0.8320 & -16.63\% & -10.31\% \\
 Nemotron-3 30B instruct & 0.4990 & 0.6436 & 0.6943 & 28.97\% & 39.14\% \\
\bottomrule
\end{tabular}
\label{tab:pass_at_n}
\end{table*}

We evaluate across eight model checkpoints: OLMo-3 7B~\citep{olmo2025olmo3} (the
Think branch) and Qwen-3 4B~\citep{yang2025qwen3technicalreport}, each on its
base, instruct, and RLVR checkpoints, together with Nemotron-3 30B in base and
post-trained variants. For each model and problem, we draw an i.i.d.\ pool of
$N{=}16$ rollouts $\{\tau_1,\ldots,\tau_N\} \sim p_\phi(\cdot\mid q)$ on a
1024-problem subset (out of $\sim1400$) of DeepMath~\citep{heLXLCW_et_al2025} at the hardest
difficulty (level $9$). The \textbf{Baseline} corresponds to standard i.i.d.\
sampling (e.g.\ GRPO) with a budget of $N$ rollouts. We measure relative
improvement via $\mathrm{Lift}(N) = (\mathrm{Pass@}N^{\text{method}} -
\mathrm{Pass@}N^{\text{i.i.d.}}) / \mathrm{Pass@}N^{\text{i.i.d.}}$, where
negative lift indicates degradation. Our central question is whether base and
aligned models, which we expect to occupy different diversity regimes, respond
differently to breadth and depth exploration.

\textbf{Stage 1: Breadth (Prefix Selection).} We first generate $N'$ short
prefixes ($< 256$ tokens) and subselect $N_e$ via a logprob-based diversity
criterion (see Appendix~\ref{app:prefix_selection}). This method accesses the
full top-$k$ logprob distribution and identifies promising prefixes, simulating
an oracle that selects the $N_e$ most distinct approaches, and requires no
additional model calls beyond the $N'$ generations. While there is no guarantee
that our heuristics cluster prefixes into genuinely distinct approaches, one
could bridge this gap with an external verifier for semantic clustering.

\textbf{Stage 2: Depth (Refinement).} Starting from the same $N_e$ selected
prefixes, we allocate the remaining budget $N - N_e$ to within-branch
refinement. We consider two \emph{alternative} depth strategies, which differ
only in how this refinement budget is spent; both operate under an identical
total token budget.

\noindent (i) \textbf{ENT}: A line of work uses the model's own uncertainty or
confidence as a signal, for instance output confidence and self-certainty~\citep{zhaoKFLS25, yoonYJEDLJY2025}, answer-probability rewards~\citep{yuJWYWCY_et_al2025}, or token-level entropy as a branching criterion~\citep{houHLLTD2025}. We adopt the latter~\citep{houHLLTD2025}: when a full rollout from a prefix is unsuccessful, we identify high-entropy tokens along its trajectory and branch at these points, to explore alternative continuations.

\noindent (ii) \textbf{SR}: We instead refine through self-critique \cite{madaanTGHG_et_al_2023}. When the
initial rollout from a prefix is incorrect, the model generates a short critique
of that rollout and then regenerates conditioned on it. The critique prompt is
specialized to the checkpoint type (base, instruct, or RLVR). At most one prior
incorrect rollout and its critique condition the next generation, so refinement
proceeds without accumulating context across iterations.

For both ENT and SR we fix the prefix-generation budget to $N' = 32$ and retain
$N_e = 6$ prefixes for refinement, leaving the remaining rollout budget for depth
exploration within the selected branches. We chose the prefix-selection rule
using a validation sweep on Qwen-3 4B over a disjoint held-out subset of
DeepMath, finding that the logprob-based diversity criterion in
Appendix~\ref{app:prefix_selection} gave the most reliable trade-off between
broad initial coverage and downstream refinement quality. We then hold these
hyperparameters fixed across all models and both methods.

\textbf{Results.} Table~\ref{tab:pass_at_n} confirms the prediction from our
framework. Two takeaways stand out. First, the regime governs the response:
across all three families, both depth strategies leave base models near their
baseline or degrade them, while mostly improving or matching the instruct and
RLVR variants. The degradation is most pronounced for the Qwen and Nemotron base
models, which are harmed by both strategies. Second, no single strategy is
uniformly best: entropy-based branching underperforms self-refinement on the
Nemotron post-trained checkpoint yet matches or beats it on the other aligned
checkpoints. The right depth signal thus depends on the model, but the
regime-level pattern, depth helps aligned models and not base ones, holds
throughout.

\textbf{Outtake.} We read this through the lens of our framework
(\S\ref{sec:theory}). For low-diversity aligned models, breadth is largely
saturated, so the dominant remaining uncertainty is epistemic depth; refinement
reduces precisely this, and a lightweight intrinsic signal suffices because the
model need only improve execution within already-covered approaches. For
high-diversity base models, the dominant uncertainty is instead epistemic
breadth: many viable approaches remain undiscovered, and spending the budget
refining a few prefixes forgoes the coverage breadth sampling would have
provided. A lightweight signal cannot recover this lost coverage and may compound
the error by refining an unpromising approach. Depth, in short, is beneficial
only once breadth is sufficiently covered; high-diversity models instead need a
larger exploration fraction or a stronger signal to identify which approaches are
worth refining.

Published work bears this out. TreeLLM$^*$~\citep{houHLLTD2025} and
Chain-in-Tree~\citep{li2025} apply MCTS with lightweight refinement signals to
instruct models, exactly where our framework predicts depth is effective.
Breadth-first approaches, by contrast, tend to target base
models~\citep{zhuangZGHLSZ2025, zhaoKFLS25}, which exhibit much larger
pass@$k$--pass@$1$ gaps, so even mild improvements over the parallel-sampling
baseline yield absolute gains. And when MCTS targets stronger improvements, it
typically relies on trained process reward
models~\citep{zhangZHYDDT2024, uesatoKKSSW_et_al2022} that provide dense,
per-step feedback, exactly the strong external signal our framework predicts is
necessary for high-diversity models.

\textbf{Remark.} We allow each model a budget of $7168$ tokens, with prompt tokens bringing up the max context length limit to $8192$. For the RLVR checkpoints and both Nemotron models, where we use an increased generation budget of $11264$ tokens, as these models tend to produce substantially longer rollouts; without the larger budget, valid answers are
extracted less reliably. The trends reported above are unchanged under this
setting. Despite this, Nemotron-post trained checkpoint manages to exhaust budget, which explains both its low performance for baseline and the dramatic improvement for alternate strategies.

\subsection{Measuring Diversity}
\label{sec:measuring-diversity}
In the previous section we examined diversity profiles through their
consequences for inference performance. In other words, diversity was studied
\emph{indirectly}, via its effect on pass rates. In this section we complement
that analysis with a direct response-level characterization of diversity. Rather
than measuring only differences in downstream results, we quantify the
variability, structure, and spread of the generated responses across
the different LLMs.

\begin{figure*}[!t]
\centering
\includegraphics[width=0.7\linewidth]{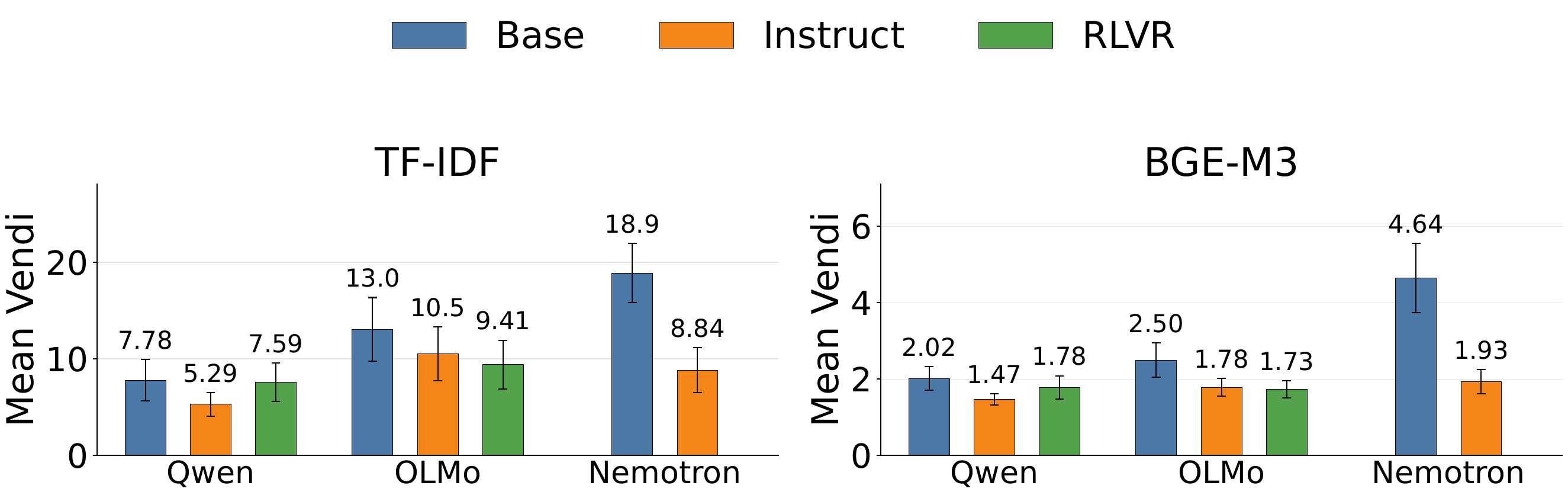}
\caption{Question-level Vendi diversity with cosine similarity. Base models
produce the largest effective response diversity, with Nemotron Base showing the
strongest contrast against its Instruct counterpart.}
\label{fig:direct-vendi}
\end{figure*}

\textbf{Diversity unit.} The key design choice is that diversity is always
measured within a fixed question. For each question $q$ and each model cell $c$,
where a cell is a model family and training variant, we collect the sampled
responses $R_{q,c}=\{r_1,\dots,r_n\}$. The sampling protocol produces up to $32$
responses per question and cell, and we analyze $26040$ traces across Qwen,
OLMo, and Nemotron. By computing every diversity statistic within a fixed
$(q,c)$ group and then averaging over questions, we avoid conflating response
diversity with differences between problem statements.

\textbf{Representing responses.} For each response $r_i$ we compute a vector
representation $x_i$, using two complementary choices. The first is a lexical
TF-IDF representation with unigram and bigram features, which is sensitive to
surface-form differences, templates, and local phrasing. The second is a dense
BGE-M3~\citep{chen2024bgem3} embedding of the reasoning trace, which captures
semantic similarity between responses. Both representations are normalized
before computing similarities.

\textbf{Vendi score.} Our primary diversity metric is the Vendi
score~\citep{friedman2023vendiscore}, which measures the effective number of
distinct response modes within a set of samples. For a fixed question and cell,
we form the $n\times n$ similarity matrix $K_{ij}=\mathrm{sim}(x_i,x_j)$. We use
cosine similarity in the main text and report an RBF kernel ($\gamma=0.5$) in
Appendix~\ref{app:clustering}. The Vendi score is the exponentiated entropy of
the normalized eigenvalues of $K/n$: a value near one means the responses are
nearly identical, while a larger value means they occupy
several distinct directions. As an alternative view, we measure the
geometric spread of responses directly via a cluster-radius statistic; this
gives consistent trends and is deferred to Appendix~\ref{app:clustering}.

\noindent Across these analyses, Base models exhibit the broadest and least
concentrated response distributions, while Instruct models are more compressed.
RLVR is not uniform across families: it increases diversity relative to Qwen
Instruct but stays close to or slightly below OLMo Instruct. The geometric
cluster-spread analysis in Appendix~\ref{app:clustering} agrees with every one
of these trends.

\textbf{Results.} Figure~\ref{fig:direct-vendi} shows that Base models generally
produce the broadest response distributions. The pattern is clearest in the
dense representation, where every Base checkpoint has a markedly higher Vendi
score than its Instruct counterpart, and the TF--IDF representation gives the
same qualitative ordering at a larger lexical scale. The Base--Instruct gap is
widest for Nemotron and narrowest for Qwen, but holds in every family. RLVR
behaves differently across families: for Qwen it partially restores diversity
relative to Instruct, while for OLMo it stays close to the Instruct condition.


\begin{figure*}[!t]
\centering
\includegraphics[width=0.8\linewidth]{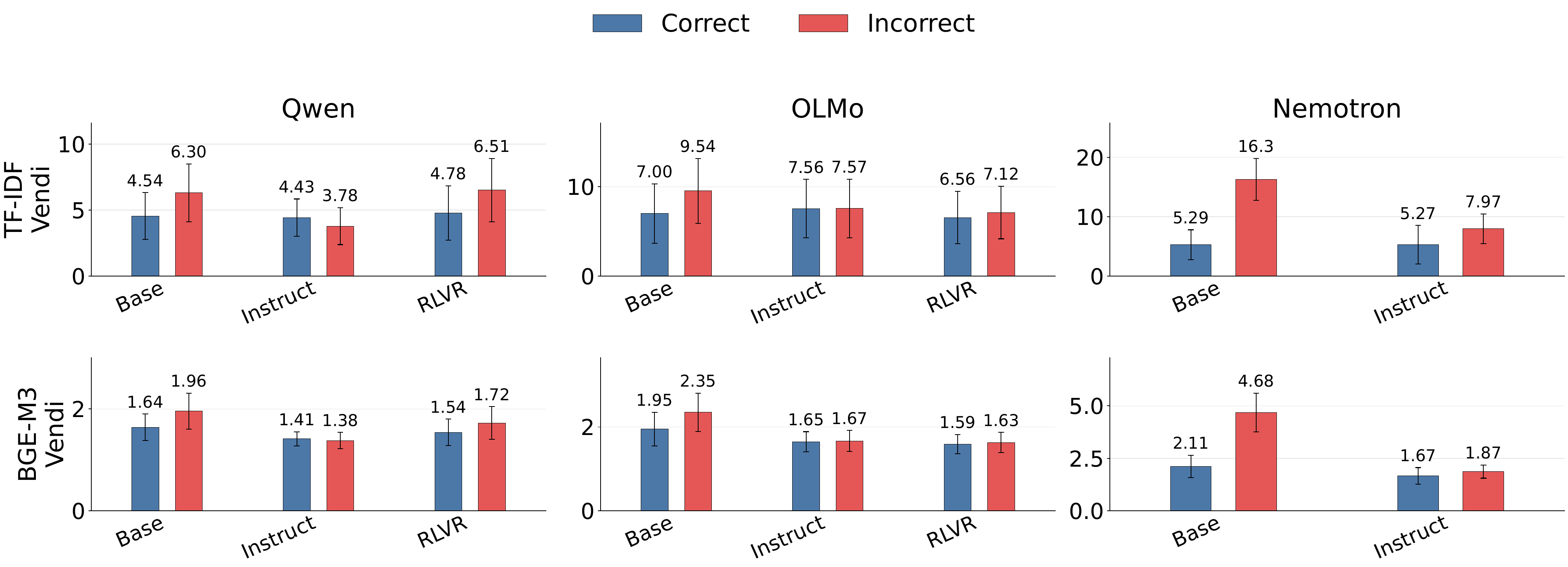}
\caption{Correctness-conditioned Vendi score. Each bar splits responses within a
question and model cell into correct and incorrect subsets, computes Vendi
separately for each, and averages over questions.}
\label{fig:correctness-vendi}
\end{figure*}


\textbf{Conditioning on correctness.} High diversity is not necessarily
desirable if it is driven by unstable or incorrect reasoning. We therefore
recompute Vendi after splitting $R_{q,c}$ into correct and incorrect subsets
using response-level correctness labels, building a separate similarity matrix
per subset whenever at least two responses are available and averaging over
questions. Figure~\ref{fig:correctness-vendi} shows that incorrect traces drive
much of the observed diversity, especially for Base models, whose incorrect
subsets are notably more diverse than correct ones. The gap is largest for Nemotron Base and present but milder for Qwen and OLMo Base. Instruct models are tighter and show little correctness gap, with some cases nearly unchanged or even slightly lower on incorrect traces. Thus even though some diversity is productive, it can indicate heterogeneous failure modes.

\textbf{Outtake.} Two findings emerge. First, the direct diversity analysis
agrees with the inference results: both the lexical (TF--IDF) and semantic
(BGE-M3) measures recover the same Base--Instruct ordering that the pass-rate
trends imply, so the diversity regimes are not just performance-centric but a property of the response distributions themselves. Second, the correctness-conditioned analysis is more revealing. Base models are consistently
more diverse on incorrect traces than on correct ones, implying a detectable signal of uncertainty. This is precisely why intrinsic-signal methods such as INTUITOR and PACR~\citep{zhaoKFLS25, yoonYJEDLJY2025} work well on base models: the model's own uncertainty is informative about where it is going wrong. For Instruct models the correctness gap is small and sometimes reversed, so the same uncertainty signal is weak or absent, perhaps hinting towards limited universality of the corresponding intrinsic-signal approaches.

\subsection{Training}
\label{sec:training}

We now ask whether the same diversity-dependent pattern governs exploration
\emph{during} training, where the policy is updated rather than held fixed. We
train Qwen-3 1.7B~\citep{yang2025qwen3technicalreport} from both its base and
post-trained checkpoints, comparing a standard GSPO~\citep{zhengLLCY_et_al2025}
baseline against a variant that uses our MCTS-entropy exploration (ENT) for
rollout collection. Both use the same difficulty-$9$ DeepMath subset of $1024$
problems with MCTS hyperparameters unchanged from \S\ref{sec:experiments}. We
hold out an additional $128$ problems for validation, evaluate greedily ($n{=}1$,
\texttt{do\_sample=False}), and report the best validation accuracy during
training.

\begin{table}[!t]
\centering
\small
\caption{Best validation accuracy during training for the GSPO baseline and the
MCTS-entropy (ENT) variant, on Qwen-3 1.7B base and post-trained checkpoints.
Starting validation accuracy is shown in brackets. Lift is the relative gain
over the corresponding baseline.}
\vspace{-4pt}
\begin{tabular}{@{}lccc@{}}
\toprule
\textbf{Model} & \textbf{Baseline} & \textbf{ENT} & \textbf{Lift}\\
\midrule
Qwen-3 1.7B base ($0.10$)     & 0.27 & 0.16 & -40.7\% \\
Qwen-3 1.7B instruct ($0.37$)  & 0.61 & 0.59 & -3.2\% \\
\bottomrule
\end{tabular}
\label{tab:training}
\end{table}

\textbf{Results.} The training results roughly reproduce the Qwen-3 4B inference
trends: MCTS exploration keeps the post-trained checkpoint within the baseline
ballpark while clearly degrading the base checkpoint, the same regime-dependent
direction. Two factors limit how much this comparison can bear. Under
group-based advantage estimation, the learning signal depends on within-group
reward variance, not raw success, so the highest-reward policy need not yield the
best training signal~\citep{xuSFK2026}; we therefore also report the critic
reward scorer over training (Fig.~\ref{fig:critic_rewards} in the Appendix), arguably a cleaner indicator of exploration, which shows the expected trend. Note that we also did not tune MCTS for training, reusing the Qwen-3 4B inference hyperparameters, so the post-trained checkpoint can plausibly match/exceed the baseline with training-specific tuning.

\section{Discussion}
    \label{sec:discussion}

We do not claim diversity is the only factor that matters. The regime is not
static: we characterize it on a fixed checkpoint, but the policy evolves during
training and the regime may shift with it. It is also a downstream property,
most directly of pretraining data, so what we measure as a model's regime is
inseparable from its data history, and two models at the same nominal training
stage may differ. Dataset composition, prompt design, and scale may likewise
confound the effects we attribute to diversity; in our own experiments, prompt
variation substantially altered the effectiveness of MCTS.

These caveats motivate a recommendation rather than a fixed recipe. We do not
expect every paper to characterize diversity as comprehensively as we do here.
But the alternative now common in the literature, demonstrating a strategy's
reasoning gains on one or two models, is too weak to establish that the strategy
generalizes. At minimum, a proposed exploration strategy should be shown to
improve pass@$k$ inference performance across several models spanning different
diversity regimes, rather than downstream performance on a single checkpoint.
This is a low bar relative to full diversity characterization, and it would
already expose the regime-dependence that single-model evaluation hides.

\section*{Limitations \& Societal Impact}

Our framework assumes a discrete set of latent approaches, an idealization of a
continuous trajectory space, and our prefix-based method assumes that early
tokens commit to an approach, which need not hold universally. The
refinement model assumes each refinement step yields the same expected
improvement, whereas in practice successive refinements likely show diminishing
returns; modeling this decay is left to future work. Empirically, our
experiments cover reasoning on a single benchmark, while diversity
profiles may differ across domains. Our depth strategies rely on lightweight
intrinsic signals; we do not test stronger external signals such as PRMs, which our analysis suggests high-diversity models may require. 

Our work studies inference- and training-time exploration in LLMs, with a
corresponding rise in energy use and carbon cost. We view diversity-aware
strategy selection partly as a mitigation, since matching the strategy to the
model's regime avoids spending this compute where it does not help, but the
methods themselves remain more expensive than standard sampling. More broadly,
techniques that strengthen LLM reasoning inherit the dual-use risks of more
capable models, and gains on benchmark mathematics need not transfer to faithful
or trustworthy reasoning in deployment.

\section*{Acknowledgments}

This research used both the DeltaAI advanced computing and data resource, which is supported by the National Science Foundation (award OAC 2320345) and the State of Illinois, and the Delta advanced computing and data resource which is supported by the National Science Foundation (award OAC 2005572) and the State of Illinois.. Delta and DeltaAI are joint efforts of the University of Illinois Urbana-Champaign and its National Center for Supercomputing Applications.

\bibliography{references}

\appendix

\begin{table*}[t]
\centering
\normalsize
\begin{tabular}{@{}
>{\raggedright\arraybackslash$}p{0.30\linewidth}<{$}
>{\raggedright\arraybackslash}p{\dimexpr0.70\linewidth-2\tabcolsep\relax}
@{}}
\toprule
\textbf{Symbol} & \textbf{Meaning} \\
\midrule

\multicolumn{2}{@{}l}{\textbf{Prefix-selection setting}}\\
\addlinespace[2pt]

N, K, L & Total rollout budget, number of selected rollouts, and prefix length. \\
x^{(i)}_{1:L} & Length-$L$ prefix of rollout $i$. \\
\ell^{(i)}_t,\, m^{(i)}_t & Chosen-token log-probability and top-1/top-2 log-probability gap at position $t$. \\
q(i),\, \hat q(i) & Prefix score and its normalized version. \\
U_t,\, \mathrm{disagree}(t) & Number of distinct tokens and disagreement score at position $t$. \\
n,\, \mathcal{T} & Number and set of anchor positions. \\
S_i & Token set of prefix $x^{(i)}_{1:L}$. \\
\mathrm{sim}_{\mathrm{Jac}}(i,j),\, \mathrm{dist}_{\mathrm{Jac}}(i,j) & Jaccard similarity and distance between prefixes $i$ and $j$. \\
\lambda,\, \alpha & Quality-diversity tradeoff and mixing weight. \\
d(i,j),\, d_{\mathrm{broad}}(i,j),\, d_{\mathrm{deep}}(i,j) & Hybrid, broad, and deep distances. \\
\mathbf{p}^{(i)}_t & Probability vector at anchor position $t$ for rollout $i$. \\
\mathrm{JSD}(\cdot,\cdot),\, D_{\mathrm{KL}}(\cdot\|\cdot) & Jensen--Shannon and Kullback--Leibler divergences. \\
w_t(i,j) & Anchor weight at position $t$. \\
\omega(s) & Step-dependent quality weight. \\

\addlinespace[4pt]
\multicolumn{2}{@{}l}{\textbf{Proof notation for Proposition~\ref{prop:mcts}}}\\
\addlinespace[2pt]

\tau_1,\ldots,\tau_N & Sampled rollouts/trajectories. \\
p_\phi(\cdot \mid q) & Model distribution conditioned on query $q$. \\
\bar{\theta} & Marginal success probability of one sampled rollout. \\
\mathcal{A}_i,\, \pi_i,\, \theta_i & Approach class, its probability, and its success probability. \\
\alpha,\, N_e,\, N_f & Exploration fraction, exploration budget, and refinement budget. \\
i^\star,\, \theta_{i^\star}^{(t)} & Best discovered approach and its stage-$t$ success probability. \\
c_{\mathrm{rand}},\, c_{i^\star}^{(t)},\, R^{(t)} & Random-sampling exponent, refinement exponent, and relative improvement ratio. \\
\eta,\, \alpha_{\min}(\eta) & Failure tolerance and minimum exploration fraction for discovery. \\
\mathrm{pass@}N_{\mathrm{rand}},\, \mathrm{pass@}N_{\mathrm{mcts}} & Success probabilities of random sampling and the MCTS-style strategy. \\

\bottomrule
\end{tabular}
\caption{Notation}
\label{tab:appendix_notation}
\normalsize
\end{table*}

\section{Notation}

The exhaustive list of notations is provided in Table~\ref{tab:appendix_notation}.

\section{Prefix Selection Methods}\label{app:prefix_selection}

\paragraph{Setting.}
For each problem, we sample $N$ independent rollouts from the model and retain only the first $L$ tokens (the \emph{prefix}) of each rollout. A selection method receives only these $N$ prefixes and must choose a subset of size $K \ll N$. Methods never use ground-truth correctness; labels are used only to evaluate $\mathrm{pass@}K$ (whether at least one selected rollout is correct).

\paragraph{Prefix features.}
For rollout $i\in[N]$, let $x^{(i)}_{1:L}$ denote its prefix tokens.
We use: (i) chosen-token log-probabilities $\ell^{(i)}_t := \log p(x^{(i)}_t \mid x^{(i)}_{<t})$, (ii) top-$k$ logprob vectors at each position (when available), and (iii) the margin $m^{(i)}_t := \log p(\text{top-1 at }t) - \log p(\text{top-2 at }t)$.

\paragraph{Quality score.}
We define the prefix sequence log-probability and its normalised form
\begin{multline}    
q(i) := \sum_{t=1}^{L} \ell^{(i)}_t, \\
\qquad
\hat q(i) := \frac{q(i)-\min_j q(j)}{\max_j q(j)-\min_j q(j)} \in [0,1].
\label{eq:quality_norm_app}
\end{multline}
We use $\hat q(i)$ when mixing quality with diversity.

\paragraph{Token-disagreement anchors.}
Let $U_t := \bigl|\{x^{(i)}_t : i\in[N]\}\bigr|$ be the number of distinct tokens observed at position $t$ across the $N$ prefixes. Define
\begin{equation}
\mathrm{disagree}(t) := \frac{U_t - 1}{N - 1} \in [0,1].
\label{eq:disagree_app}
\end{equation}
Given an anchor budget $n$, we select the $n$ positions with the highest $\mathrm{disagree}(t)$ and denote the set by $\mathcal{T}$.

\paragraph{Jaccard similarity on prefix tokens.}
For rollout $i$, let $S_i := \{x^{(i)}_1,\dots,x^{(i)}_L\}$ be the set of tokens appearing in its prefix. Define
\begin{multline}
\mathrm{sim}_{\mathrm{Jac}}(i,j) := \frac{|S_i \cap S_j|}{|S_i \cup S_j|},\\
\qquad
\mathrm{dist}_{\mathrm{Jac}}(i,j) := 1-\mathrm{sim}_{\mathrm{Jac}}(i,j).
\label{eq:jaccard_app}
\end{multline}

\subsection{Random@$K$ (baseline)}\label{app:random}
\paragraph{Method.}
Select $K$ rollouts uniformly at random from the $N$ samples. This provides the null baseline for whether prefix information is useful.

\subsection{MMR (Maximal Marginal Relevance)}\label{app:mmr}
MMR greedily balances selecting \emph{high-quality} prefixes while avoiding redundancy with already-selected prefixes.

\paragraph{Greedy rule.}
Let $S$ be the selected set (initially empty) and $R$ the remaining indices.
At each step, choose
\begin{equation}
i^* \in \arg\max_{i\in R}\;\Bigl[
\lambda \,\hat q(i)\;-\;(1-\lambda)\,\max_{j\in S}\mathrm{sim}_{\mathrm{Jac}}(i,j)
\Bigr],
\label{eq:mmr_app}
\end{equation}
add $i^*$ to $S$, and repeat until $|S|=K$.
(When $S$ is empty, the similarity term is taken as $0$.)

\paragraph{Key hyperparameter: $\lambda$ (quality--diversity tradeoff).}
The parameter $\lambda\in[0,1]$ controls how strongly the method prioritises high-probability prefixes versus novelty:
\begin{itemize}
    \item $\lambda \approx 1$: \emph{quality-dominant}. The method approaches selecting the top-$K$ prefixes by $\hat q(i)$, using diversity only as a weak tie-breaker.
    \item $\lambda \approx 0$: \emph{diversity-dominant}. The method prioritises reducing redundancy, even if it must accept lower-probability prefixes.
\end{itemize}
In our sweeps, smaller $\lambda$ typically helps in high-diversity regimes (where many rollouts follow the same incorrect approach), while larger $\lambda$ can help in low-diversity regimes (where the dominant approach is more often viable).

\subsection{Adaptive Dispersion (token-disagreement anchors)}\label{app:adisp}
Adaptive Dispersion is a logprob-aware diversity selection method designed to distinguish \emph{true approach forks} from superficial variation. It has three components:
(1) find anchor positions where rollouts actually diverge,
(2) define a hybrid distance that combines distributional differences at anchors with broad token-level differences,
and (3) greedily select a set that transitions from quality to diversity as it fills the $K$ slots.

\subsubsection{Step 1: choose anchor positions}
Compute disagreement scores \eqref{eq:disagree_app} and select the top $n$ positions:
\begin{equation}
\mathcal{T} := \text{top-}n \text{ positions by } \mathrm{disagree}(t).
\label{eq:anchors_app}
\end{equation}

\paragraph{Key hyperparameter: $n$ (number of anchors).}
The integer $n$ controls how many positions are treated as decision points.
\begin{itemize}
    \item Small $n$: anchors focus on the most prominent early forks; this is robust but may miss secondary divergences.
    \item Large $n$: anchors capture finer-grained branching structure but can include noisy positions, which makes distance estimates less stable.
\end{itemize}
A useful rule of thumb is that $n$ should be small relative to $L$ and also not so large that most prefixes become unique ``by chance'' at the anchor positions.

\subsubsection{Step 2: hybrid distance}
Define a distance between prefixes $i$ and $j$ as
\begin{equation}
d(i,j) := \alpha \, d_{\mathrm{deep}}(i,j) + (1-\alpha)\, d_{\mathrm{broad}}(i,j),
\label{eq:hybrid_dist_app}
\end{equation}
with $\alpha\in[0,1]$.

\paragraph{Broad distance.}
We use Jaccard distance:
\begin{equation}
d_{\mathrm{broad}}(i,j) := \mathrm{dist}_{\mathrm{Jac}}(i,j).
\label{eq:dbroad_app}
\end{equation}

\paragraph{Deep distance (distributional at anchors).}
At each anchor position $t\in\mathcal{T}$, let $\mathbf{p}^{(i)}_t$ be the probability vector obtained by applying $\mathrm{softmax}$ to the stored top-$k$ logprob vector for rollout $i$ at position $t$ (renormalised over the top-$k$ support). We measure distributional disagreement via Jensen--Shannon divergence:
\begin{multline}
\mathrm{JSD}\!\left(\mathbf{p},\mathbf{r}\right)
:= \frac{1}{2} D_{\mathrm{KL}}\!\left(\mathbf{p}\,\|\,\mathbf{m}\right)
+ \frac{1}{2} D_{\mathrm{KL}}\!\left(\mathbf{r}\,\|\,\mathbf{m}\right), \\
\qquad \mathbf{m}=\frac{\mathbf{p}+\mathbf{r}}{2}.
\label{eq:jsd_app}
\end{multline}
We weight anchors by inverse margin:
\begin{equation}
w_t(i,j) := \frac{1}{1 + \frac{1}{2}\bigl(m^{(i)}_t + m^{(j)}_t\bigr)}.
\label{eq:weight_app}
\end{equation}
Then
\begin{equation}
d_{\mathrm{deep}}(i,j)
:= \frac{\sum_{t\in\mathcal{T}} w_t(i,j)\,\mathrm{JSD}\!\left(\mathbf{p}^{(i)}_t,\mathbf{p}^{(j)}_t\right)}
{\sum_{t\in\mathcal{T}} w_t(i,j)}.
\label{eq:ddeep_app}
\end{equation}

\paragraph{Key hyperparameter: $\alpha$ (deep vs.\ broad distance).}
The mixing parameter $\alpha\in[0,1]$ trades off:
\begin{itemize}
    \item $\alpha \approx 1$: distance is dominated by distributional differences at anchors (captures subtle ``state-of-mind'' differences even when sampled tokens match).
    \item $\alpha \approx 0$: distance reduces toward token-level Jaccard over the whole prefix (captures broad stylistic/lexical differences).
\end{itemize}
Empirically, larger $\alpha$ tends to help when anchor positions reliably correspond to genuine approach forks; smaller $\alpha$ can be more stable when token-level differences are spread across many positions.

\subsubsection{Step 3: greedy max--min selection with an adaptive schedule}
Adaptive Dispersion builds $S$ greedily using a max--min diversity term together with a step-dependent quality weight.
Let $s\in\{0,1,\dots,K-1\}$ be the selection step. Define a linear schedule
\begin{multline}
\omega(s) := \omega_{\mathrm{init}}\Bigl(1-\frac{s}{K-1}\Bigr) + \omega_{\mathrm{final}}\frac{s}{K-1},\\
0\le \omega_{\mathrm{final}}\le \omega_{\mathrm{init}}\le 1.
\label{eq:schedule_app}
\end{multline}
At step $s$, select
\begin{equation}
i^* \in \arg\max_{i\in R}\;\Bigl[
\omega(s)\,\hat q(i) + \bigl(1-\omega(s)\bigr)\,\min_{j\in S} d(i,j)
\Bigr],
\label{eq:adisp_score_app}
\end{equation}
add $i^*$ to $S$, and repeat.
When $S$ is empty, the diversity term is omitted and we pick the highest-quality prefix.

\paragraph{Key hyperparameters: $\omega_{\mathrm{init}}, \omega_{\mathrm{final}}$ (quality schedule).}
The schedule in \eqref{eq:schedule_app} controls how quickly the method transitions from \emph{exploitation} to \emph{diversification}:
\begin{itemize}
    \item Larger $\omega_{\mathrm{init}}$: the first few picks prioritise a strong ``anchor'' prefix (high $\hat q$).
    \item Smaller $\omega_{\mathrm{final}}$: later picks become close to pure diversity via the max--min term $\min_{j\in S} d(i,j)$.
\end{itemize}
A wide gap $\omega_{\mathrm{init}}-\omega_{\mathrm{final}}$ yields a stronger shift toward diversity as the set fills; a narrow gap keeps the method relatively quality-focused throughout.

\paragraph{Interaction with $K$.}
Because $\omega(s)$ depends on $K$, the same $(\omega_{\mathrm{init}},\omega_{\mathrm{final}})$ can behave differently for different $K$ values. For small $K$, the schedule has fewer steps and thus less opportunity to ``cool down'' from quality to diversity; for larger $K$, the later selections are more purely diversity-driven.

\subsection{Method taxonomy (summary table)}\label{app:method_taxonomy}
Table~\ref{tab:method_taxonomy} summarises the methods used in the main text by the signals they rely on.

\begin{table*}[h]
\centering
\small
\caption{Taxonomy of prefix-selection methods used in the main text. ``Top-$k$'' indicates whether the method requires per-position top-$k$ logprob vectors (beyond chosen-token logprobs).}
\label{tab:method_taxonomy}
\begin{tabular}{@{}lccc@{}}
\toprule
\textbf{Method} & \textbf{Quality signal} & \textbf{Diversity signal} & \textbf{Uses top-$k$?} \\
\midrule
Random@$K$ & none & none & No \\
MMR & prefix logprob $\hat q(i)$ & Jaccard over prefix tokens & No \\
Adaptive Dispersion & prefix logprob $\hat q(i)$ & token-disagreement anchors + hybrid distance & Yes \\
\bottomrule
\end{tabular}
\end{table*}

\subsection{Hyperparameter sweeps (values and reporting)}\label{app:hyperparam_sweeps}
We sweep small grids and report the best-performing variant per model and $K$ in Table~1 of the main text.

\paragraph{MMR.}
We sweep $\lambda$ over a small set (e.g., $\{0.3,0.5,0.7\}$). Lower values emphasise novelty, while higher values emphasise selecting high-logprob prefixes first.

\paragraph{Adaptive Dispersion.}
We sweep:
\begin{itemize}
    \item $n$ (anchors): a small set such as $\{6,8,10\}$,
    \item $\alpha$ (deep weight): a small set such as $\{0.3,0.5,0.7\}$,
    \item $(\omega_{\mathrm{init}},\omega_{\mathrm{final}})$ (schedule endpoints): a small set such as $(0.7,0.1)$ and $(0.5,0.1)$.
\end{itemize}
Across these variants, the method spans from \emph{quality-first} (large $\omega_{\mathrm{init}}$) to \emph{diversity-first} (small $\omega_{\mathrm{init}}$), and from \emph{anchor-heavy} (large $\alpha$) to \emph{token-overlap-heavy} (small $\alpha$).

\section{Proof of Proposition~\ref{prop:mcts}}\label{app:proof_mcts}

\begin{proof}
We prove the ratio condition and the feasibility statement.

Under i.i.d.\ sampling, $\tau_1,\ldots,\tau_N \overset{\mathrm{i.i.d.}}{\sim} p_\phi(\cdot\mid q)$ and each draw
succeeds with marginal probability $\bar{\theta}:=\Pr(\text{success})$. Hence
\begin{multline}
\mathrm{pass@}N_{\mathrm{rand}} = 1-(1-\bar{\theta})^N
=1-\exp\!\big(-N c_{\mathrm{rand}}\big),\\
c_{\mathrm{rand}}:=-\log(1-\bar{\theta}).
\end{multline}
If we further partition trajectory space into approaches $\{\mathcal{A}_i\}_{i=1}^m$ and define
$\pi_i:=\Pr(\tau\in\mathcal{A}_i)$ and $\theta_i:=\Pr(\text{success}\mid \tau\in\mathcal{A}_i)$, then by the
law of total probability $\bar{\theta}=\sum_{i=1}^m \pi_i\theta_i$.

Let $\alpha:=N_e/N$ so $N_f=(1-\alpha)N$. Under the oracle identification assumption, exploitation is
restricted to $i^*$ and consists of $N_f$ independent attempts, each succeeding with probability
$\theta_{i^*}^{(t)}$. Therefore,
\begin{align*}
\mathrm{pass@}N_{\mathrm{mcts}}
=1-(1-\theta_{i^*}^{(t)})^{N_f}
\\=1-\exp\!\big(-N_f c_{i^*}^{(t)}\big),\\
c_{i^*}^{(t)}:=-\log(1-\theta_{i^*}^{(t)}).
\end{align*}
Thus $\mathrm{pass@}N_{\mathrm{mcts}}\ge \mathrm{pass@}N_{\mathrm{rand}}$ is equivalent to
\begin{multline}
\exp\!\big(-(1-\alpha)N c_{i^*}^{(t)}\big)\le \exp\!\big(-N c_{\mathrm{rand}}\big)
\iff \\
(1-\alpha)c_{i^*}^{(t)} \ge c_{\mathrm{rand}}
\iff
\alpha \le 1-\frac{c_{\mathrm{rand}}}{c_{i^*}^{(t)}}.
\end{multline}
With $R^{(t)}:=c_{i^*}^{(t)}/c_{\mathrm{rand}}$, this becomes $\alpha \le 1-1/R^{(t)}$.

During exploration, each of the $N_e=\alpha N$ samples lands in approach $i^*$ with probability $\pi_{i^*}$,
so
\[
P(\mathrm{discover}\ i^*) = 1-(1-\pi_{i^*})^{N_e} = 1-(1-\pi_{i^*})^{\alpha N}.
\]
To ensure $P(\mathrm{discover}\ i^*)\ge 1-\eta$, it suffices that
$(1-\pi_{i^*})^{\alpha N}\le \eta$, i.e.,
\[
\alpha \ge \frac{\log(1/\eta)}{N\,\log\!\big(\frac{1}{1-\pi_{i^*}}\big)}=:\alpha_{\min}(\eta),
\]
and using $\log\!\big(\frac{1}{1-x}\big)\approx x$ for small $x$ gives
$\alpha_{\min}(\eta)\approx \frac{\log(1/\eta)}{N\pi_{i^*}}$.
Finally, an $\alpha$ satisfying both discovery and outperformance exists iff
$\alpha_{\min}(\eta)\le 1-\frac{1}{R^{(t)}}$, equivalently $R^{(t)}\ge \frac{1}{1-\alpha_{\min}(\eta)}$.
\end{proof}
\section{Approach Switching: Statements and Proofs}
\label{app:proof_approach_switching}

 \paragraph{Setup.} Fix a problem $q$, model $\phi$, and budget $B$; we suppress the dependence on $(\phi,B)$ and write $\pi_i$ and $\theta_i$. There are $m$
latent approaches. The active approach evolves as a Markov chain over $[m]$ with
switching matrix $S_\rho \in \mathbb{R}^{m\times m}$ given by
$(S_\rho)_{ij} = 1-\rho$ for $i=j$ and $\rho/(m-1)$ otherwise. We write
$\bar\theta_u := \frac{1}{m}\sum_i \theta_i$ for the uniform-approach success
rate and $\bar\theta_\pi := \sum_i \pi_i \theta_i$ for the no-switching rate. Throughout this section, all distributions over approaches are treated as
row vectors. Let $\mathbf{1}\in\mathbb{R}^m$ be the all-ones column vector and
let
\[
    u^\top := \frac{1}{m}\mathbf{1}^\top
\]
denote the uniform row distribution over approaches.

\subsection{Proof of Theorem~\ref{thm:unconditional_switching}}

\begin{theorem}[Unconditional switching]
\label{thm:unconditional_switching}
Let $A_t \in [m]$ be the active approach at step $t$ with $A_0 \sim \pi$, evolving
by $S_\rho$ for $H$ steps. Let $\lambda_\rho = 1 - \frac{m}{m-1}\rho$. The success
probability after $H$ switching opportunities is
\[
P_H(\rho) = \bar\theta_u + \lambda_\rho^H(\bar\theta_\pi - \bar\theta_u).
\]
If $0 < \rho \le (m-1)/m$, then $P_H(\rho)$ interpolates between $\bar\theta_\pi$
and $\bar\theta_u$, and for any $H \ge 1$, switching improves over no switching
iff $\bar\theta_u > \bar\theta_\pi$.
\end{theorem}

\begin{proof}
Let $S_\rho\in\mathbb{R}^{m\times m}$ be the transition matrix over latent
approaches, with
\[
(S_\rho)_{ij}
=
\begin{cases}
1-\rho, & i=j,\\
\rho/(m-1), & i\neq j .
\end{cases}
\]
Each row of $S_\rho$ sums to one. Each column also sums to one, since
\[
    (1-\rho)+(m-1)\frac{\rho}{m-1}=1 .
\]
Thus $S_\rho$ is doubly stochastic and the uniform distribution $u$ is
stationary.

Define
\[
    \lambda_\rho := 1-\frac{m}{m-1}\rho ,
    \qquad
    P := \mathbf{1}u^\top .
\]
The matrix $P$ has every row equal to $u^\top$ and is a projection:
$P^2=P$. Moreover, $IP=PI=P$. We first show that
\[
    S_\rho = \lambda_\rho I+(1-\lambda_\rho)P .
\]
Indeed, for $i\neq j$,
\[
    \bigl[\lambda_\rho I+(1-\lambda_\rho)P\bigr]_{ij}
    =
    \frac{1-\lambda_\rho}{m}
    =
    \frac{\rho}{m-1},
\]
while for $i=j$,
\[
\begin{aligned}
    \lambda_\rho+\frac{1-\lambda_\rho}{m}
    &=
    1-\frac{m\rho}{m-1}
    +
    \frac{\rho}{m-1}  \\
    &= 1-\rho .
\end{aligned}
\]
Therefore the identity holds.

Since $P$ is a projection and commutes with $I$, the power of $S_\rho$ is
\[
    S_\rho^H
    =
    P+\lambda_\rho^H(I-P).
\]
Equivalently, for any initial row distribution $\pi^\top$,
\[
\begin{aligned}
    \pi^\top S_\rho^H
    &=
    \pi^\top P+\lambda_\rho^H\pi^\top(I-P) \\
    &=
    u^\top+\lambda_\rho^H(\pi^\top-u^\top).
\end{aligned}
\]
Let $\theta=(\theta_1,\ldots,\theta_m)^\top$. The success probability after
$H$ switching opportunities is
\[
    P_H(\rho)=\pi^\top S_\rho^H\theta .
\]
Substituting the previous display gives
\[
\begin{aligned}
    P_H(\rho)
    &=
    u^\top\theta
    +
    \lambda_\rho^H(\pi^\top-u^\top)\theta  \\
    &=
    \bar\theta_{\mathrm{unif}}
    +
    \lambda_\rho^H
    \left(
        \bar\theta_\pi-\bar\theta_{\mathrm{unif}}
    \right).
\end{aligned}
\]
This proves the claimed closed form.

It remains to compare with no switching. When $\rho=0$,
$\lambda_\rho=1$, so $P_H(0)=\bar\theta_\pi$. If
$0<\rho\le (m-1)/m$, then $\lambda_\rho\in[0,1)$, and for any $H\ge 1$,
\[
\begin{aligned}
    P_H(\rho)-P_H(0)
    &=
    \Bigl(1-\lambda_\rho^H\Bigr)
    \left(
        \bar\theta_{\mathrm{unif}}-\bar\theta_\pi
    \right).
\end{aligned}
\]
Since $1-\lambda_\rho^H>0$, switching improves success if and only if
$\bar\theta_{\mathrm{unif}}>\bar\theta_\pi$.
\end{proof}

\subsection{Proof of Theorem~\ref{thm:failure_conditioned_switching}}

\begin{theorem}[Failure-conditioned switching]
\label{thm:failure_conditioned_switching}
Suppose the model switches only after a failed attempt: at attempt $t$ with
active approach $i$, the attempt succeeds with probability $\theta_i$, and on
failure the approach updates by $S_\rho$. Let $F = \operatorname{diag}(1-\theta_1,
\ldots, 1-\theta_m)$. The probability of solving within $H$ attempts is
\[
P^{\mathrm{bt}}_H(\rho) = 1 - \pi^\top (F S_\rho)^H \mathbf{1}.
\]
\end{theorem}

\begin{proof}
Let $r_t^\top\in\mathbb{R}^m$ denote the row vector whose $i$th entry is the
probability that, after $t$ failed attempts, the trajectory is still unresolved
and the active approach is $i$. Initially,
\[
    r_0^\top=\pi^\top .
\]
Given active approach $i$, an attempt fails with probability $1-\theta_i$.
Thus, after the failure event but before switching, the unresolved mass vector
is $r_t^\top F$, where
\[
    F=\operatorname{diag}(1-\theta_1,\ldots,1-\theta_m).
\]
Conditional on failure, the next active approach is then updated according to
$S_\rho$. Hence
\[
    r_{t+1}^\top = r_t^\top F S_\rho .
\]
By induction,
\[
    r_H^\top = \pi^\top (F S_\rho)^H .
\]
The probability that all $H$ attempts fail is the total unresolved mass after
$H$ failed attempts:
\[
    \Pr(\text{fail for all }H\text{ attempts})
    =
    r_H^\top \mathbf{1}
    =
    \pi^\top(FS_\rho)^H\mathbf{1}.
\]
Taking the complement yields
\[
    P_H^{\mathrm{bt}}(\rho)
    =
    1-\pi^\top(FS_\rho)^H\mathbf{1}.
\]
This proves the theorem.
\end{proof}

\subsection{Proof of Proposition~\ref{prop:two_step_switching_gain}}

\begin{proof}
For $H=2$, condition on the initial approach $i$. The first attempt fails with
probability $1-\theta_i$. After this failure, the second attempt uses approach
$i$ with probability $1-\rho$ and uses each $j\neq i$ with probability
$\rho/(m-1)$. Therefore, the conditional probability of two consecutive
failures given $A_0=i$ is
\[
\begin{aligned}
    &(1-\theta_i)
    \Bigg[
        (1-\rho)(1-\theta_i) \\
    &\qquad\qquad
        +
        \frac{\rho}{m-1}
        \sum_{j\neq i}(1-\theta_j)
    \Bigg].
\end{aligned}
\]
Averaging over $A_0\sim\pi$ gives
\[
\begin{aligned}
    1-P_2^{\mathrm{bt}}(\rho)
    =
    \sum_{i=1}^m
    \pi_i(1-\theta_i)
    \Bigg[
        (1-\rho)(1-\theta_i) \\
        +
        \frac{\rho}{m-1}
        \sum_{j\neq i}(1-\theta_j)
    \Bigg].
\end{aligned}
\]
When $\rho=0$, the model stays with the same approach after failure, so
\[
    1-P_2^{\mathrm{bt}}(0)
    =
    \sum_{i=1}^m
    \pi_i(1-\theta_i)^2 .
\]
Subtracting the two success probabilities gives
\[
\begin{aligned}
&P_2^{\mathrm{bt}}(\rho)-P_2^{\mathrm{bt}}(0) \\
&=
\sum_{i=1}^m
\pi_i(1-\theta_i)
\Bigg[
    (1-\theta_i)
    -
    (1-\rho)(1-\theta_i) \\
&\qquad\qquad\qquad
    -
    \frac{\rho}{m-1}
    \sum_{j\neq i}(1-\theta_j)
\Bigg] \\
&=
\rho
\sum_{i=1}^m
\pi_i(1-\theta_i)
\Bigg[
    (1-\theta_i)\\
    &\qquad\qquad\qquad-
    \frac{1}{m-1}
    \sum_{j\neq i}(1-\theta_j)
\Bigg].
\end{aligned}
\]
Now define
\[
    \bar\theta_{-i}
    :=
    \frac{1}{m-1}\sum_{j\neq i}\theta_j .
\]
Then
\[
\begin{aligned}
    &(1-\theta_i)
    -
    \frac{1}{m-1}
    \sum_{j\neq i}(1-\theta_j) \\
    &\qquad =
    \frac{1}{m-1}\sum_{j\neq i}\theta_j-\theta_i
    =
    \bar\theta_{-i}-\theta_i .
\end{aligned}
\]
Substituting this identity yields
\[
    P_2^{\mathrm{bt}}(\rho)-P_2^{\mathrm{bt}}(0)
    =
    \rho
    \sum_{i=1}^m
    \pi_i(1-\theta_i)
    (\bar\theta_{-i}-\theta_i).
\]
Because $\rho\ge 0$, the gain is positive if and only if the displayed sum is
positive.
\end{proof}

\subsection{Proof of Corollary~\ref{cor:one_good_approach}}

\begin{corollary}
\label{cor:one_good_approach}
Suppose one approach $g$ has success probability $\theta_{\mathrm{hi}}$ and all
others have $\theta_{\mathrm{lo}} < \theta_{\mathrm{hi}}$, and let $p = \pi_g$.
Then one-step failure-conditioned switching improves success iff
\[
p < \frac{1-\theta_{\mathrm{lo}}}{(1-\theta_{\mathrm{lo}}) + (m-1)(1-\theta_{\mathrm{hi}})}.
\]
\end{corollary}

\begin{proof}
Let $\Delta:=\theta_{\mathrm{hi}}-\theta_{\mathrm{lo}}>0$. For the good
approach $g$,
\[
    \bar\theta_{-g}-\theta_g
    =
    \theta_{\mathrm{lo}}-\theta_{\mathrm{hi}}
    =
    -\Delta .
\]
For any weak approach $i\neq g$,
\[
\begin{aligned}
    \bar\theta_{-i}-\theta_i
    &=
    \frac{\theta_{\mathrm{hi}}+(m-2)\theta_{\mathrm{lo}}}{m-1}
    -
    \theta_{\mathrm{lo}}  \\
    &=
    \frac{\Delta}{m-1}.
\end{aligned}
\]
Using Proposition~\ref{prop:two_step_switching_gain}, switching improves
success if and only if
\[
    -p(1-\theta_{\mathrm{hi}})\Delta
    +
    (1-p)(1-\theta_{\mathrm{lo}})
    \frac{\Delta}{m-1}
    >0 .
\]
Since $\Delta>0$, this is equivalent to
\[
    (1-p)(1-\theta_{\mathrm{lo}})
    >
    p(m-1)(1-\theta_{\mathrm{hi}}).
\]
Rearranging,
\[
\begin{aligned}
    1-\theta_{\mathrm{lo}}
    &>
    p\Bigl[
        1-\theta_{\mathrm{lo}}
        +(m-1)(1-\theta_{\mathrm{hi}})
    \Bigr],
\end{aligned}
\]
and therefore
\[
    p
    <
    \frac{1-\theta_{\mathrm{lo}}}
    {(1-\theta_{\mathrm{lo}})
    +(m-1)(1-\theta_{\mathrm{hi}})} .
\]
This proves the corollary.
\end{proof}

\section{Diversity: Cluster-Spread Analysis}
\label{app:clustering}
As an alternative to the Vendi score, we measure the geometric spread of
responses directly. For each question and cell, we treat the responses as an
assumed cluster, compute the centroid $\mu_{q,c}=\frac{1}{n}\sum_{i=1}^{n}x_i$,
and report the root-mean-squared radius
$\mathrm{RMS}=\bigl(\frac{1}{n}\sum_{i=1}^{n}\|x_i-\mu_{q,c}\|_2^2\bigr)^{1/2}$,
averaged over questions. This is not a learned clustering procedure: the
clusters are defined by the cell itself. We also report the RBF-kernel Vendi
score ($\gamma=0.5$) referenced in the main text.

\begin{figure*}[!t]
\centering
\includegraphics[width=0.7\linewidth]{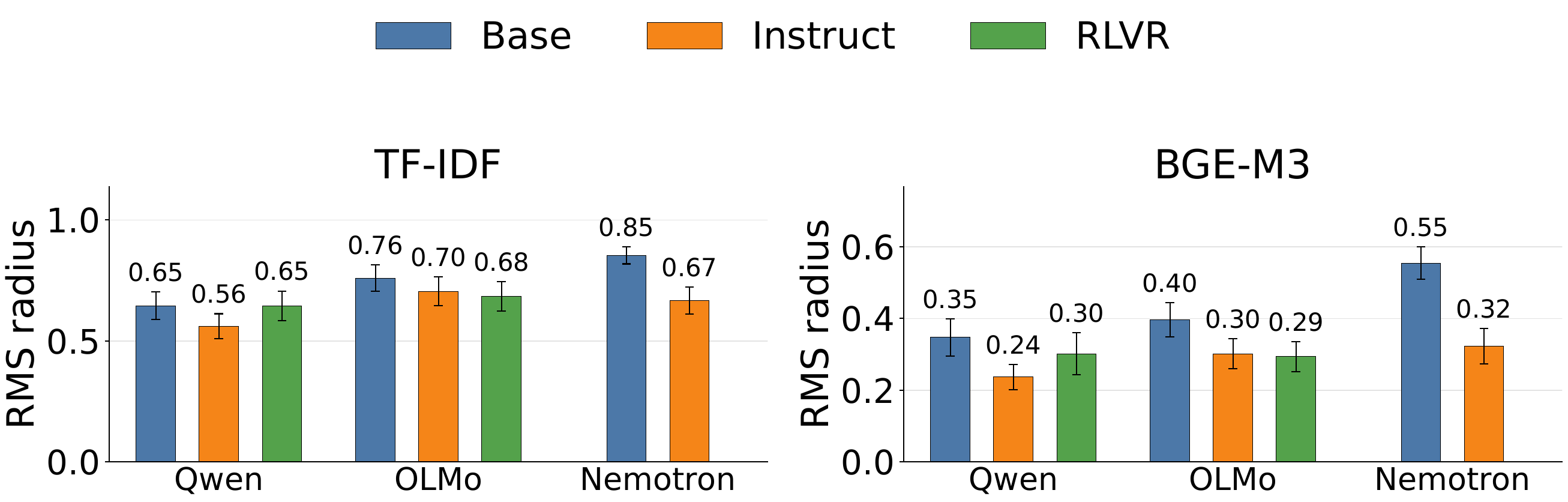}
\caption{Assumed-cluster spread. Base cells form wider response clusters than
Instruct cells, showing that instruction tuning compresses both mode diversity
and geometric spread.}
\label{fig:cluster-rms}
\end{figure*}

The cluster-spread analysis in Figure~\ref{fig:cluster-rms} supports the same
conclusion as the Vendi analysis: Base models are consistently the widest
clusters. In BGE-M3 space the RMS radius is $0.35$ for Qwen Base versus $0.24$
for Qwen Instruct, $0.40$ versus $0.30$ for OLMo, and $0.56$ versus $0.32$ for
Nemotron, with TF--IDF showing the same ordering at a lexical scale. Pairwise
cosine statistics agree: Nemotron Base has the lowest BGE-M3 mean pairwise
cosine ($0.68$) while Nemotron Instruct is substantially tighter ($0.89$). Thus
instruction tuning reduces both the number of effective response modes and the
geometric spread of sampled traces.

\begin{figure*}[!t]
\centering
\includegraphics[width=0.9\linewidth]{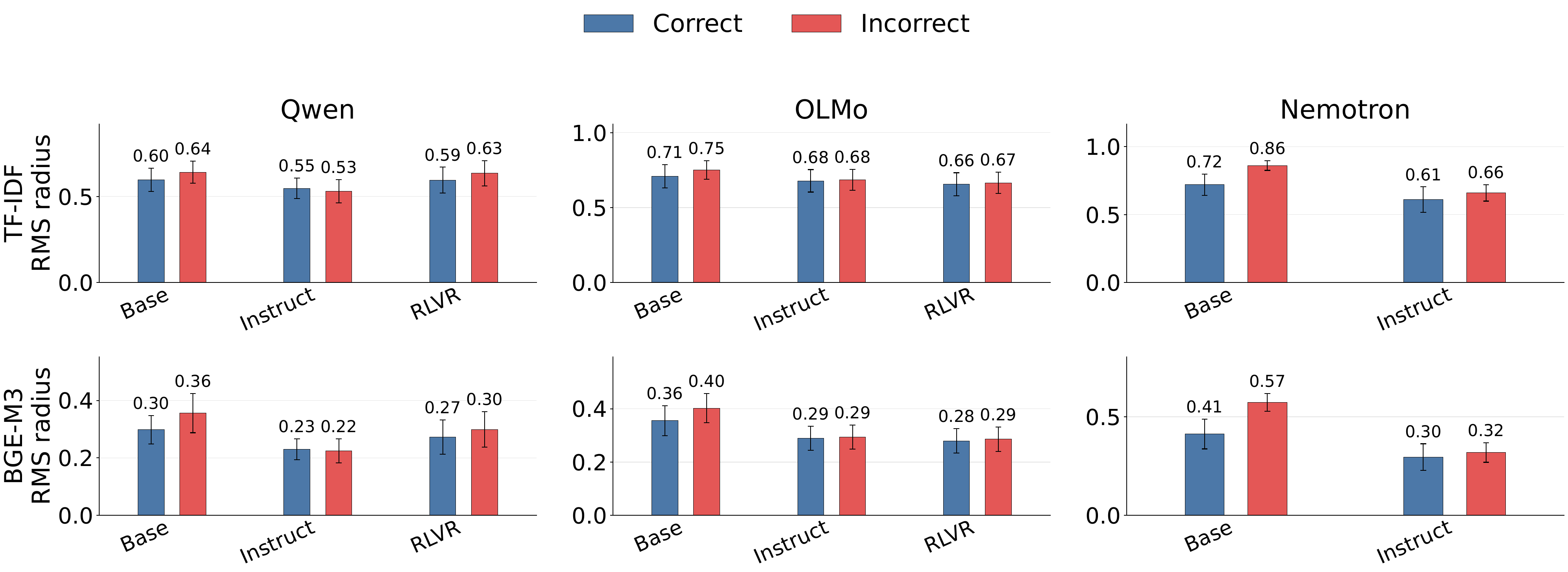}
\caption{Correctness-conditioned clustering. Incorrect responses are typically
more spatially dispersed than correct responses, especially for Base models.}
\label{fig:correctness-rms}
\end{figure*}

We also condition the cluster spread on correctness, treating correct and
incorrect responses as two predefined clusters whenever at least two responses
are available in a subset. Figure~\ref{fig:correctness-rms} shows that incorrect
responses are more spread out than correct ones, especially for Base models. In
BGE-M3 space, Nemotron Base has RMS radius $0.41$ on correct traces and $0.57$
on incorrect ones; Qwen Base increases from $0.30$ to $0.36$ and OLMo Base from
$0.36$ to $0.40$. TF--IDF shows the same direction (Nemotron Base $0.72$ to
$0.86$, OLMo Base $0.71$ to $0.75$, Qwen Base $0.60$ to $0.64$). Instruct and
RLVR cells are much more stable across the correctness split, with differences
near zero for Qwen Instruct, OLMo Instruct, and OLMo RLVR. Together with the
correctness-conditioned Vendi scores, this indicates that in the
highest-diversity cells, incorrect traces are not only more mode-diverse but
also occupy a wider region of representation space.

\section{Training Reward Tracking}

\begin{figure}[!t]
\centering
\includegraphics[width=\linewidth]{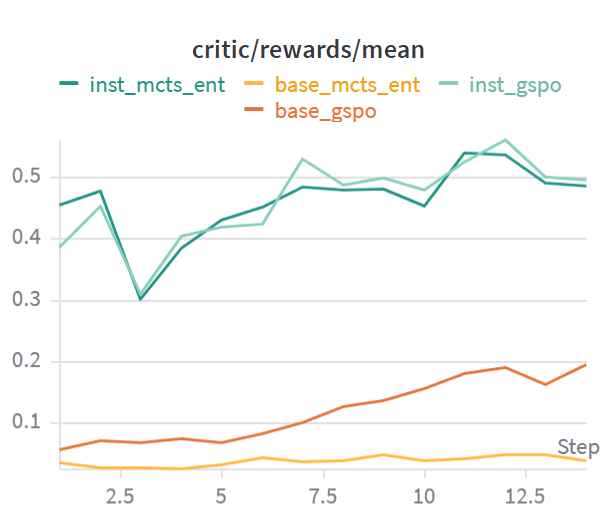}
\caption{Critic rewards offer a direct signal on quality of exploration; the plot verifies the expected trend on Qwen family i.e. the performance of Entropy-based MCTS closely matches standard sampling for the instruct model, while significantly degrades on the base model.}
\label{fig:critic_rewards}
\end{figure}

\section{Direct Diversity Analysis Details}
\label{app:direct-diversity}

\paragraph{Data and cells.}
The clean response set contains 260,040 traces after preprocessing. The analyzed cells are Qwen Base, Qwen Instruct, Qwen RLVR, OLMo Base, OLMo Instruct, OLMo RLVR, Nemotron Base, and Nemotron Instruct. Nemotron RLVR is not present in the data and is therefore omitted from all figures. The number of unique cleaned question texts is 1023 for Qwen cells, 1024 for OLMo cells, and 1023 for Nemotron Base and Instruct. Most question-cell groups contain 32 responses; a small number contain fewer responses because of missing generations or trace filtering.

\paragraph{Trace preprocessing.}
Each rollout file contains a problem, answer metadata, a list of sampled responses, and response-level correctness labels. We flatten these files into one row per response trace. The flattened row stores the cleaned question, cleaned response, compact response text, model family, variant, run name, question id, and correctness label. Responses shorter than the minimum cleaned length are removed. The diversity analyses use the compact response field so that the metrics focus on the generated reasoning trace rather than irrelevant serialization artifacts.

\paragraph{TF--IDF representation.}
The lexical representation uses a word-level TF--IDF vectorizer with unigram and bigram features, sublinear term frequency, minimum document frequency filtering, and $\ell_2$ normalization. The representation is fitted once on all cleaned response traces in the analysis set, after which per-question similarity matrices are formed by selecting the rows belonging to each question-cell group. This representation intentionally preserves surface-level differences in phrasing and reasoning templates.

\paragraph{Dense representation.}
The semantic representation uses BGE-M3 dense embeddings. Embeddings are computed for the compact response strings and normalized before similarity computation. The dense representation is less sensitive to exact wording than TF--IDF, so it gives a complementary view of whether responses differ semantically rather than merely lexically. In practice, BGE-M3 scores are lower than TF--IDF scores because many lexically distinct traces are embedded close together semantically.

\paragraph{Similarity kernels.}
The main text reports cosine similarity. For each question-cell group with response vectors $x_1,\ldots,x_n$, the cosine similarity matrix is
\[
K_{ij}=\max(0,\cos(x_i,x_j)),
\]
with $K_{ii}=1$. We also compute an RBF kernel based on cosine distance:
\[
d^2_{ij}=2-2K_{ij}, \qquad
K^{\mathrm{RBF}}_{ij}=\exp(-0.5d^2_{ij}).
\]
The RBF kernel compresses the absolute Vendi scale but preserves the same interpretation: larger values indicate a more diverse set of responses for the same question.

\paragraph{Vendi implementation.}
For every question-cell group, we compute eigenvalues of $K/n$, clip small numerical negatives to zero, normalize them into $p_i$, and report $\exp(-\sum_i p_i\log p_i)$. This score can be read as an effective number of response modes. We do not compute a single Vendi score over all responses pooled across questions; instead, we compute one score per question-cell group and average over questions.

\paragraph{Correctness-conditioned implementation.}
For correctness-conditioned Vendi and RMS radius, we split responses within each question-cell group into correct and incorrect subsets. If a subset contains at least two responses, we build its own similarity matrix for Vendi and its own centroid for RMS radius. The final correct and incorrect values for a model cell are means over the questions for which the corresponding subset is available.

\paragraph{Cluster-spread implementation.}
For cluster spread, the response vectors for a question-cell group define one assumed cluster. We compute the centroid and the RMS radius around that centroid. We also compute mean pairwise cosine similarity over off-diagonal response pairs. A larger RMS radius indicates greater geometric dispersion; a larger mean pairwise cosine indicates a tighter response set. These metrics are again averaged over questions.

\paragraph{Base--Instruct separation.}
For families with both Base and Instruct variants, we also compare the two assumed clusters question by question. We compute the cosine distance between the Base and Instruct centroids and normalize the squared centroid distance by the pooled within-cluster variance. We also compute a silhouette-style score using cosine distance and the predefined Base/Instruct labels. These diagnostics characterize whether instruction tuning shifts the center of the response distribution, not only its spread.

\begin{figure*}[t]
\centering
\includegraphics[width=0.92\linewidth]{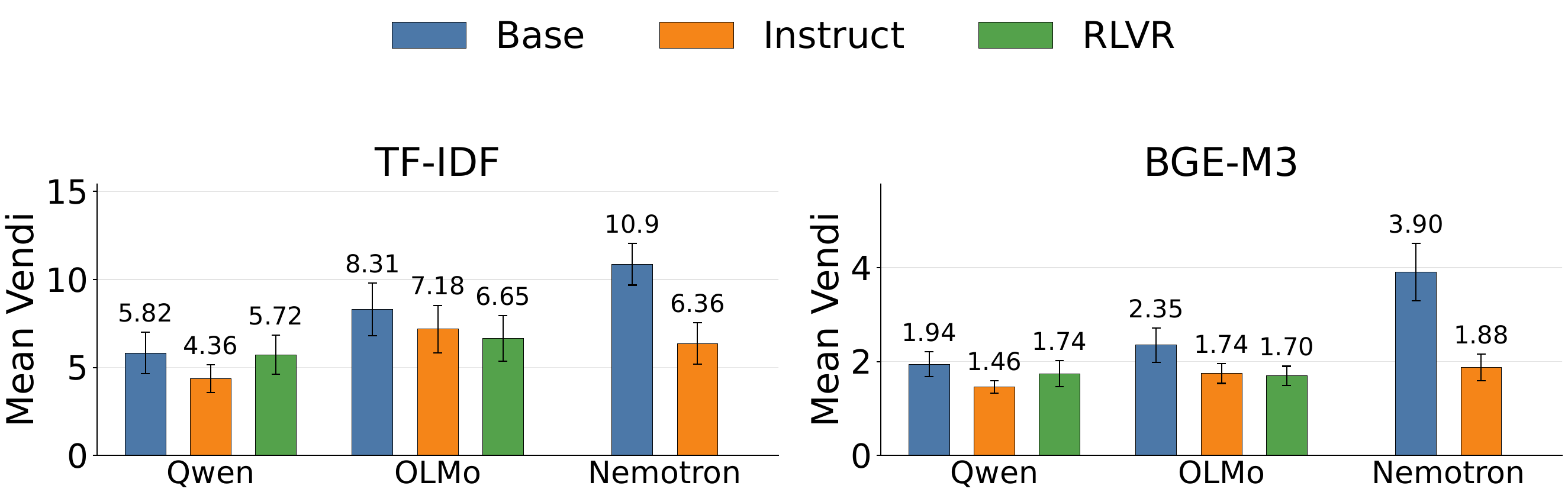}
\caption{The RBF kernel compresses the Vendi scale but preserves the main ordering: Base cells remain the most diverse.}
\label{fig:app-vendi-rbf}
\end{figure*}

\begin{figure*}[t]
\centering
\includegraphics[width=0.92\linewidth]{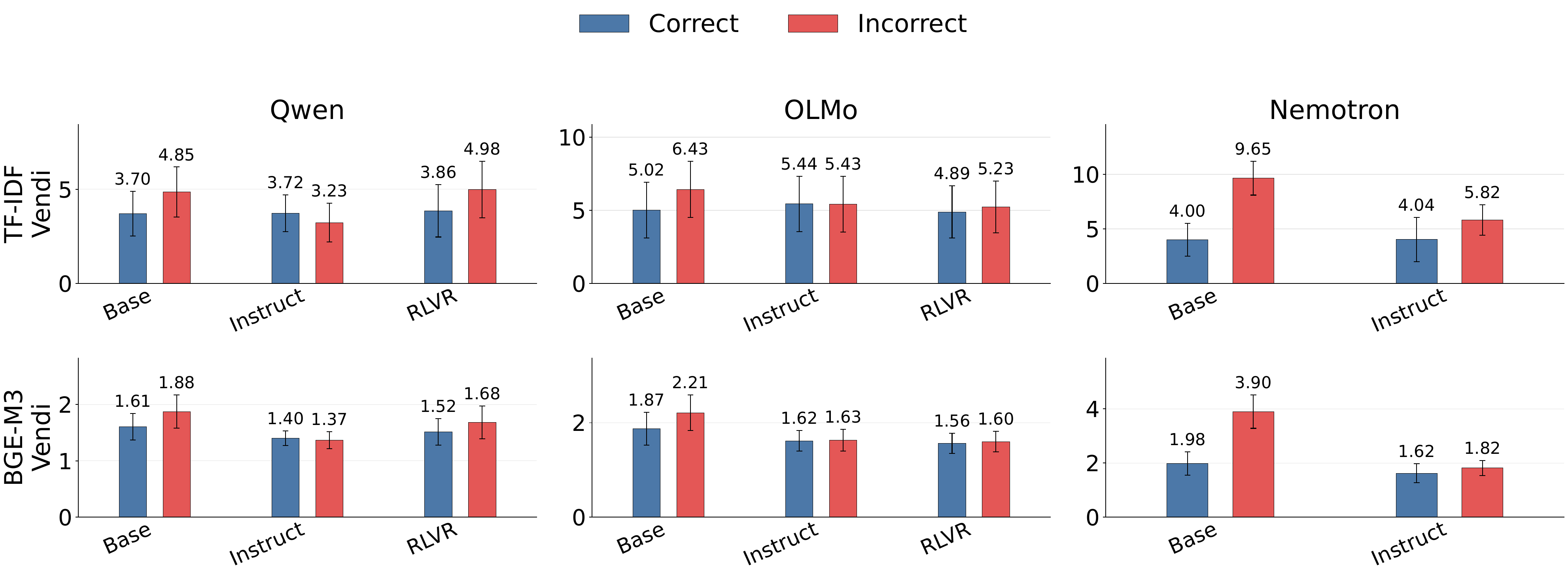}
\caption{Under the RBF kernel, incorrect traces still account for the largest correctness-conditioned diversity gaps.}
\label{fig:app-correctness-vendi-rbf}
\end{figure*}

\begin{figure*}[t]
\centering
\includegraphics[width=0.92\linewidth]{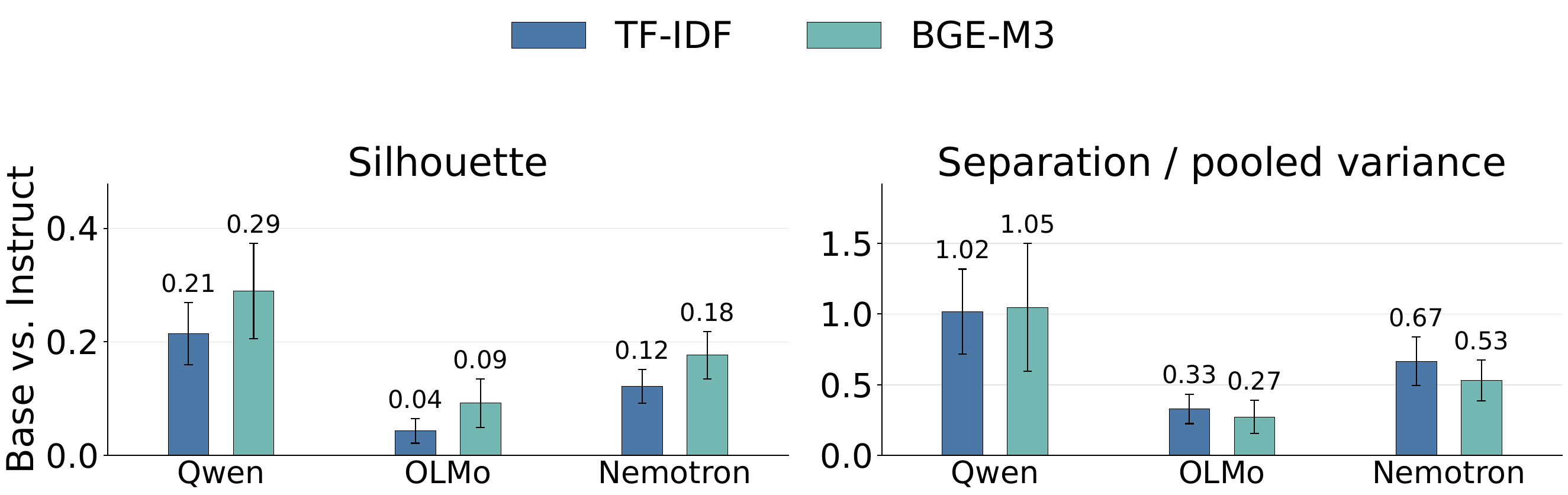}
\caption{Qwen shows the clearest Base--Instruct geometric separation, OLMo the weakest, and Nemotron an intermediate but still visible shift.}
\label{fig:app-base-instruct-separation}
\end{figure*}

\end{document}